\documentclass[10pt,dvipsnames]{article}
\usepackage[utf8]{inputenc} %
\usepackage[T1]{fontenc}    %
\usepackage{url}            %
\usepackage[top=1in, bottom=1in, left=1in, right=1in]{geometry}
\usepackage{mathtools}
\usepackage{amsmath, amssymb, amsthm, bbm}
\usepackage{physics}
\usepackage{colorblind}
\usepackage{graphicx}
\usepackage{hyperref}       %
\usepackage{caption}
\usepackage{tensor}
\usepackage[most]{tcolorbox}
\usepackage{empheq}
\usepackage{enumitem}
\usepackage{float}
\usepackage{graphicx}
\usepackage{booktabs}       %
\usepackage{amsfonts}       %
\usepackage{nicefrac}       %
\usepackage{microtype}      %
\usepackage{cleveref}
\usepackage{natbib}
\usepackage{tikz}
\usepackage{xfrac}
\usepackage{algpseudocode}
\usepackage{algorithm}
\usepackage{authblk}
\usepackage{dsfont}
\usepackage{subcaption}

\newcommand{\bx}{\mathbf{x}}%
\newcommand{\bz}{\mathbf{z}}%
\newcommand{\bpi}{\boldsymbol{\pi}}%

\newcommand{\bW}{\mathbf{W}}%

\newcommand{\bb}{\mathbf{b}}%
\newcommand{\ba}{\boldsymbol{\alpha}}%

\newcommand{\bI}{\mathbf{I}}%
\newcommand{\bun}{\mathbf{1}}%
\newcommand{\be}{\mathbf{e}}%

\DeclareMathOperator*{\argmax}{argmax}
 
\setcitestyle{authoryear,round,citesep={;},aysep={,},yysep={;}}
\hypersetup{colorlinks, breaklinks=true, urlcolor=orange, linkcolor=blue, citecolor=blue}

\theoremstyle{plain}
\newtheorem{theorem}{Theorem}[section]
\newtheorem{proposition}{Proposition}
\newtheorem{lemma}{Lemma}
\newtheorem{corollary}{Corollary}
\theoremstyle{definition}
\newtheorem{definition}{Definition}

\theoremstyle{remark}

\crefname{assumption}{Assumption}{Assumptions}

\title{The Well-Tempered Classifier: \\ Some Elementary Properties of Temperature Scaling}

\author[1]{Pierre-Alexandre Mattei}
\author[2]{Bruno Loureiro}

\affil[1]{\small Université Côte d'Azur, Inria, CNRS, LJAD, France}
\affil[2]{\small Département d’Informatique, École Normale Supérieure - PSL, CNRS, France}

\date{\today}

\allowdisplaybreaks
\begin{document}

\maketitle
\begin{abstract}
  Temperature scaling is a simple method that allows to control the uncertainty of probabilistic models. It is mostly used in two contexts: improving the calibration of classifiers and tuning the stochasticity of large language models (LLMs). In both cases, temperature scaling is the most popular method for the job. Despite its popularity, a rigorous theoretical analysis of the properties of temperature scaling has remained elusive. We investigate here some of these properties. For classification, we show that increasing the temperature increases the uncertainty in the model in a very general sense (and in particular increases its entropy). However, for LLMs, we challenge the common claim that increasing temperature increases diversity. Furthermore, we introduce two new characterisations of temperature scaling. The first one is geometric: the tempered model is shown to be the information projection of the original model onto the set of models with a given entropy. The second characterisation clarifies the role of temperature scaling as a submodel of more general linear scalers such as matrix scaling and Dirichlet calibration: we show that temperature scaling is the only linear scaler that does not change the hard predictions of the model.
\end{abstract}

\section{Introduction}
\label{sec:intro}
Reliable uncertainty quantification in deep learning remains an open problem, poised between strong theoretical limits \citep{foygel2021limits} and a flurry of proposed solutions (see e.g. \citealp{silva2023classifier,angelopoulos2023conformal, ulmer2023prior, papamarkou2024position} for presentations of the different paradigms at play).

In a seminal paper, \citet{guo2017calibration} advocated the use of an old and very simple uncertainty quantification method called \emph{temperature scaling}. This technique employs a single scalar parameter (the temperature) to tune the confidence of a trained neural network. In the complex landscape of uncertainty quantification, temperature scaling imposed itself as an immensely popular method. Its use goes from calibrating models in industrial machine learning---it is implemented in the Scikit-learn library \citep{scikit-learn}---to controlling inference in large language models (LLMs)---its use is pervasive in virtually all LLMs, and is explicitly mentioned in the technical reports of GPT-4 \citep{achiam2023gpt}, Gemini \citep{team2023gemini}, DeepSeek \citep{liu2024deepseek}, and Mistral \citep{rastogi2025magistral}. Often, other uncertainty quantification techniques use temperature scaling as a building block (e.g. \citealp{berta2025rethinking,gibbs2025conformal}).

In spite of this popularity, there is a surprising lack of theoretical investigations of temperature scaling. Exceptions include \citet{clarte2023double,clarte2023expectation}, who studied in particular its asymptotic behaviour under model misspecification, \citet{dabah2025temperature} who looked at the interplay between temperature scaling and conformal prediction, and \citet{berta2025structured}, who studied in which cases it can (and cannot) be optimal. This relative absence may be due to the fact that the model is quite simple, and resembles many well-studied models in statistics, machine learning, and statistical physics. Nevertheless, we believe a thorough yet elementary inspection of temperature scaling to be both timely and interesting. After reviewing and revisiting temperature scaling in Section \ref{sec:whatisTS} (both for classification and LLMs), our \textbf{main contributions} are:
\begin{itemize}
    \item In Section \ref{sec:entropy}, we show that for classification, increasing the temperature has the expected effect of increasing the uncertainty of the model. However, we highlight that this may not be the case for LLMs.
    \item We provide in Section \ref{sec:geometry} a geometric interpretation of temperature scaling as an information projection: a tempered model is the model closest to the original one that has the required level of entropy.
    \item An important property of temperature scaling is that it is \emph{accuracy-preserving:} it does not change the ordering of the classes. We show in Section \ref{sec:consistency} that is actually a defining property, in the sense that it is the only accuracy-preserving linear scaler.
\end{itemize}

\textbf{Notation.} Vectors are denoted by boldfaces.
We denote by $\Delta_{K}=\{\bpi\in [0,1]^{K} |  \pi_{1} + \ldots + \pi_K=1\}$ the $K$-simplex and by $H$ the Gibbs-Shannon entropy of a discrete distribution.

\section{What is temperature scaling?}
\label{sec:whatisTS}
We start with a pretrained model $f: \mathcal{X} \rightarrow \mathbb{R}^K$ where $K$ is the number of classes --- for instance, a neural network trained on a classification problem. The output of this model are the logits, i.e. the softmax preactivations $\bz = f(\bx)$. Standard predictive probabilities are obtained by then applying a softmax layer:
\begin{equation}
	p(y | \bz) = \text{Categorical}(y |\bpi) = \pi_y, 
\end{equation}
with
\begin{equation}
\boldsymbol{\pi} = \text{Softmax} (\bz ) \in \Delta_{K},
\end{equation}
where $y \in \{1,\ldots ,K\}$ is the class label and
\begin{equation}
	\text{Softmax} (\bz ) = \left( \frac{e^{z_1}}{e^{z_1}+ \ldots + e^{z_K}}, \ldots, \frac{e^{z_K}}{e^{z_1}+ \ldots + e^{z_K}} \right).
\end{equation}
Here, we mean ‘‘model" in a very broad sense: while $f$ could be a standard classifier (with features $\bx$ and label $y$), it could also be, for instance, an autoregressive language model (in that case, $\bx$ would be the previous tokens, $y$ the next token(s), and $K$ the size of the vocabulary).

Each $\pi_k\in[0,1]$ corresponds to the probability that the model assigns to class $k$. Sometimes, one may not be too happy with these probabilities. In particular, in the context of classification tasks with neural networks, a widely reported issue is that these probabilities are generally overconfident, leading to poor calibration \citep{guo2017calibration, minderer2021revisiting} and difficult distillation \citep{hinton2015distilling}. 
For LLMs, it is desirable to let users sharpen or smoothen these probabilities to give them control of the amount of stochasticity of the generations, also referred to as ``diversity'' in this context. Temperature scaling is a simple and efficient method that allows us to slightly alter the probabilities in order to meet one of these goals. 

The idea is to multiply all logits by a learnable parameter $\beta > 0$, called the \emph{inverse temperature} in analogy to a Boltzmann-Gibbs distributions in statistical physics. This leads to a new conditional distribution
\begin{equation}
    p_{\beta}(y | \bz) = \text{Categorical}(y |\boldsymbol{\pi}_\beta ) \text{ with } \boldsymbol{\pi}_\beta = \text{Softmax} (\beta \bz ).
\end{equation}
Since $\beta$ is constrained to be positive, the coefficients of $\beta \bz$ are ordered in the same way as those of $\bz$. Thus, the hard predictions of $p_{\beta}$ will be identical to the ones of the original model. In particular, this implies that temperature scaling is \emph{accuracy-preserving: it does not change the accuracy of the initial model.} This feature was already highlighted by \citet{guo2017calibration}, and is particularly relevant for modern deep classifiers, which generally excel in terms of accuracy, but are often overconfident. As we will see in Section \ref{sec:consistency}, this is in fact a defining characteristic of temperature scaling.

\paragraph{What if we don't have the logits?} We may not have access to the logits $\bz$, but only to the class probabilities $\boldsymbol{\pi}$. A natural idea is to replace the logits by the log-probabilities $\log \bpi$. Since the softmax is not invertible, $\log \bpi$ is not necessarily equal to $\bz$ and there may be a risk of loosing partial information when using $\log \bpi$ instead of $\bz$. Fortunately, this is not the case and the two approaches are equivalent. More precisely, as detailed in Appendix  \ref{app:logitsvsloprobs},
\begin{align}
    \text{Softmax} (\beta \bz ) &= \text{Softmax} (\beta \log \bpi ) = \left( \frac{\pi_1^{\beta}}{\pi_1^{\beta}+ \ldots + \pi_K^{\beta}}, \ldots, \frac{\pi_K^{\beta}}{\pi_1^{\beta}+ \ldots + \pi_K^{\beta}} \right).
    \label{eq:logitsvssoftmax}
\end{align}
Because of this equivalence, some papers (e.g. \citealp{guo2017calibration}) define temperature scaling using the logits, and some using the log-probabilities (e.g. \citealp{berta2025structured}).

Another simple way of write Equation \eqref{eq:logitsvssoftmax} is $p_\beta \propto p^\beta$.

\subsection{Why is $\beta$ called the inverse temperature?} 
\label{sec:statphys}
A tempered softmax probabilistic classifier admits a natural interpretation in terms of a statistical physics problem in the canonical ensemble with $K$ discrete states and inverse temperature $\beta$, where each state $k\in \{1, \ldots, K \}$ has energy given by the (negative) logit $E_{k}=-z_{k}$. From this perspective, the tempered categorical distribution is just the Boltzmann-Gibbs distribution of the system: 
\begin{align}
    p_{\beta}(y|\bz)=\frac{e^{-\beta E_{y}}}{\mathcal{Z}(\beta, \bz)},
\end{align}
where the denominator
\begin{equation}
    \mathcal{Z}(\beta, \bz) = e^{\beta z_1}+ \ldots + e^{\beta z_K}.
\end{equation}
is called the \emph{partition function}.
Another equivalent point of view would be to see $p_{\beta}$ as an exponential family distribution with sufficient statistic $z_y$.

The partition function has a few remarkable yet elementary properties that will considerably simplify computations related to temperature scaling. In particular, the derivatives of its logarithm are related to the moments of the logit of a random label.

\begin{lemma} We have
\label{prop:logpartition}
   \begin{align}
        &\frac{d \log \mathcal{Z}(\beta, \bz)}{d \beta} =  \mathbb{E}_{y \sim p_\beta(y | \bx)} [ z_y ], \\
        &\frac{d^2 \log \mathcal{Z}(\beta, \bz)}{d \beta^2} =  \textup{Var}_{y \sim p_\beta(y | \bx)} ( z_y ).
    \end{align}
\end{lemma}
Variants of Lemma \ref{prop:logpartition} are standard results in statistical physics (see e.g. \citealp{kardar2007statistical}) and in the theory of conditional exponential families (see e.g. \citealp{domke2020moment}). We give a proof in our specific context in Appendix \ref{app:logpartition}. These results will prove useful to study the properties of the likelihood, that can be written, for a single data point, as
\begin{equation}
    \log p_\beta (y| \bx) = \beta z_{y} - \log \mathcal{Z}(\beta, \bz).
\end{equation}
From the statistical physics perspective, this is simply the free energy cost for the system to occupy state $y$.

\subsection{How to tune the temperature?}

\subsubsection{Calibrating classifiers}
In the context of calibration, to tune $\beta$, it is customary to use a fresh set of labelled data $(\bx_1,y_1) \ldots, (\bx_n,y_n)$, called a \emph{calibration set}. Once the corresponding predictions $\boldsymbol{\pi}_1,   \ldots ,\boldsymbol{\pi}_n$ have been computed (this involves a forward pass through the neural network),  the inverse temperature is tuned by minimising the cross-entropy loss (or equivalently, maximising the log-likelihood) of the one-parameter model $p_\beta$ over the calibration set. This leads to the estimate
\begin{equation}
	\hat{\beta} = \text{argmin}_{\beta > 0} \mathcal{L}(\beta), \text{ with } \mathcal{L}(\beta) = -\frac{1}{n} \sum_{i=1}^n \log  p_{\beta} ( y_i | \bx_i). %
\end{equation}
It is sensible to tune the temperature using the cross-entropy for several reasons. First, when $n \rightarrow \infty$, $p_{\beta}$ will converge to the closest model (in terms of Kullback-Leibler divergence) to the true data-generating process (see e.g. \citealp{van2000asymptotic}, Example 2.25). Moreover, the cross-entropy is a proper scoring rule, which means that maximising it should lead to well-calibrated predictions \citep{blasiok2023does}. Typically, cross-entropy is also used for training the classifier.  Employing the same loss for training and calibration suggests an apparent paradox: since we used the same loss for training and calibration, why was the initial model poorly calibrated? 
Classical statistical intuition suggests that one possible reason is that the original model, typically an overparametrised deep neural network, was likely subject to severe overfitting, which leads to overconfidence. On the other hand, temperature scaling involves fitting a \emph{one-parameter model}, which is unlikely to overfit. However, modern networks can overfit without loosing in generalisation error, a phenomenon known as \emph{benign overfitting}. In particular, asymptotic results on high-dimensional logistic regression have shown a non-monotonic behaviour of the calibration with overparametrisation \citep{bai2021don,clarte2023double, clarte2023theoretical}, suggesting a more complex picture for the source of overconfidence in overparametrised models. 

Minimising the cross-entropy over the calibration set is not a particularly difficult task, as we shall see below. Interestingly, its gradient has a simple moment-matching interpretation for the cross-entropy, and its Hessian can be written as the variance of the logits of the original model.

\begin{proposition}
\label{prop:likelihood}
	The function $\mathcal{L}$ is smooth and convex. For any $\beta$, its first and second derivative satisfy
	\begin{align}
		\label{eq:grad-hess}
		 \frac{d \mathcal{L}(\beta)}{d\beta} &=\frac{1}{\beta} \left( \underbrace{-\frac{1}{n} \sum_{i=1}^n  \log p_\beta (y_i | \bx_i)}_\textup{cross-entropy}  -  \underbrace{\left(-\frac{1}{n} \sum_{i=1}^n\mathbb{E}_{y \sim p_\beta(y | \bx_i)} [ \log p_\beta(y |\bx_i)]\right)} _\textup{expected cross-entropy}\right),\\
    \frac{d^2 \mathcal{L} (\beta)}{d\beta^2}  &= \frac{1}{n} \sum_{i=1}^n \textup{Var}_{y \sim p_\beta(y | \bx_i)} ( z_{iy} ).
    \end{align}    
    
	If there exists at least one prediction not entirely made of ties, i.e. if there exists $i \in \{1, \ldots, n\}, y,y' \in \{1, \ldots, K\}$ such that $z_{iy} \neq z_{iy'}$, then $\mathcal{L}$ is strictly convex. Otherwise, if all predictions are ties, $\mathcal{L}$ is constant.

	If there is at least one misclassification in the calibration set, then  $\mathcal{L}$ has a unique minimiser. Otherwise, if the calibration accuracy is one, $\mathcal{L}$ is strictly decreasing: in a sense, the optimal temperature is thus zero. 
\end{proposition}
The proof is a straightforward consequence of Lemma \ref{prop:logpartition} and is provided in Appendix \ref{app:likelihood}.
This result shows that, under very mild conditions, $\mathcal{L}$ will be easy to optimise. The initial implementation of \citet{guo2017calibration} used L-BFGS \citep{liu1989limited} to minimise $\mathcal{L}$ but simpler algorithms like the bisection method (advocated by \citealp{berta2025rethinking}) are also good options. Temperature scaling was recently implemented as an option of the \texttt{CalibratedClassifierCV} function, within the popular Scikit-learn library \citep{scikit-learn}, using an algorithm of \citet{brent1973algorithms} through the \texttt{minimize\_scalar} function of Scipy \citep{scipy}. Given the simplicity of the problem, any well-tuned optimiser should be able to find $\hat{\beta}$ efficiently. The only potentially problematic case is the one of a classifier with zero classification error on the calibration set. This may happen either because the model has perfectly learned a truly deterministic relationship between $x$ and $y$ (in this case, the initial model was \emph{not confident enough} and temperature scaling will fix this by making the model fully deterministic), or, more realistically, because the calibration set was too small (which is cause for concern, because temperature scaling will then artificially increase overconfidence). A simple way to alleviate the latter is to add the constraint $\beta \leq 1 $ which implies that the tempered model cannot be more confident than the original one. Another option is to slightly regularise the optimisation problem: for instance, \citet{berta2025rethinking} suggest using a form of Laplace smoothing. 

Equation $\eqref{eq:grad-hess}$ leads to an interesting interpretation of maximum likelihood as moment-matching. Indeed, finding a $\hat{\beta}$ such that $\nabla \mathcal{L}(\hat{\beta}) = 0$ means that the expected value of the cross-entropy of the tempered model, matches its actual observed value. This is related to the moment-matching interpretation of maximum-likelihood in exponential families (see e.g. \citealp{domke2020moment}). Although minimising the cross-entropy is arguably the most popular choice for tuning $\beta$, any other moment-matching technique can be chosen in principle. For instance, we could imagine matching the accuracy instead of the cross-entropy. This would be equivalent to the \emph{expectation consistency} of \citet{clarte2023expectation}, who, based on the observation that for the optimal classifier the expected accuracy should match the expected confidence, proposed to find a $\hat{\beta}$ such that
\begin{align}
\label{eq:EC}
    \quad\frac{1}{n}\sum\limits_{i=1}^{n}p_{\hat{\beta}}(y_{i}|\bx_{i}) = \frac{1}{n}\sum\limits_{i=1}^{n}\boldsymbol{1}(\hat{y}(\bx_{i})=y_{i}),
\end{align}
where $\hat{y}(\bx)=\argmax_{y} p_{\beta}(y|\bx)$ is the hard label predictor. Note that since tempering does not change the accuracy, the right-hand side is independent of $\beta$, and this condition can be imposed by optimising the difference over $\beta>0$ as in temperature scaling. \citet{clarte2023expectation} has shown that expectation consistency is competitive with minimising the cross-entropy and has proven that it can outperform it in synthetic classification tasks.

Other alternatives to maximum likelihood have been explored. \citet{mozafari2018attended} designed a loss function tailored for small calibration sets. \citet{shen2024thermometer} used variational Bayesian inference, in an LLM context.  

\subsubsection{Calibrating language models}

There are two different notions of calibration for LLMs. The first one, \emph{next token calibration}, is the simply the calibration of the autoregressive task of predicting the next token. The second one, \emph{semantic calibration}, refers to calibration of question-answering tasks.
In general, language models are somewhat well-calibrated \emph{for both tasks} after their pretraining phase, but this calibration is jeopardised by both fine-tuning based on reinforcement learning and chain-of-thought reasoning \citep{nakkiran2025trained}.

Temperature scaling has been used to tackle this lack of calibration, typically using calibration data in a fashion similar to what we described for classification  \citep{shen2024thermometer, xie2024calibrating}.

\subsubsection{Manual tuning for language models}

In contrast with the calibration setting described above, temperature scaling is most often used in language models as a manual decoding hyperparameter, rather than a parameter learned from labelled calibration data. In this context, the temperature is typically exposed to the user and adjusted heuristically to control the perceived diversity of the generated text. Alternative techniques to control text diversity are also used (e.g. nucleus sampling, \citealp{holtzman2020curious}).

 \citet{peeperkorn2024temperature} empirically showed that changing the temperature is very weakly correlated with text creativity. We will see in Section \ref{sec:entropy} that temperature can have counter-intuitive effects on LLMs.

\subsection{Extensions of temperature scaling}

\subsubsection{Linear scalers}

 In their seminal paper, \citet{guo2017calibration} also introduced a more general scaler called \emph{matrix scaling}, defined as 
\begin{equation}
\label{eq:matscaling}
    \tilde{p}_{\bW,\bb}(y | \bx) = \text{Categorical}(y |\text{Softmax} (\bW \bz + \bb )),
\end{equation}
where $\bW \in \mathbb{R}^{K \times K}$ and  $\bb \in \mathbb{R}^K$ are trainable parameters. For binary classification, matrix scaling reduces to Platt's \citeyearpar{platt1999probabilistic} scaling, and in the general multiclass case, it is equivalent to training a logistic regression model on top of the logits of the pretrained model. A similar model called \emph{Dirichlet calibration} was introduced by \citet{kull2019beyond}. Motivated by a generative perspective (seeing the probabilities of the initial model as a mixture of Dirichlet distributions), they proposed using log-probabilities instead of logits:
\begin{equation}
\label{eq:dirichlet}
    \check{p}_{\bW,\bb}(y | \bx) = \text{Categorical}(y |\text{Softmax} (\bW \log \bpi + \bb )).
\end{equation}

While more flexible than temperature scaling, matrix scaling is generally trickier to fit on small calibration sets because of its much larger number of parameters. Another issue is that it is possible that the calibrated model will predict different labels than the original one, i.e. the mode of $\tilde{p}_{\bW,\bb}$ may not be the same as the one of $p$, which may lead to a drop in accuracy. Partly due to these issues, \citet{guo2017calibration} found temperature scaling empirically superior to matrix scaling. Several papers have consequently proposed specific regularisation techniques to avoid overfitting, dating back to \citet{platt1999probabilistic} who used a form of soft labelling (to alleviate the fact that he was training the scaler using the training set). More recently, \citet{kull2019beyond} and \citet{berta2025structured} proposed regularisation strategies tailored for matrix scaling. Their conclusions were that, when carefully regularised, it can be competitive with temperature scaling.

\subsubsection{Nonlinear scalers}

A common limitation of temperature/matrix scaling and Dirichlet calibration is that they are linear in the logits or log-probabilities of the initial model. A simple way to turn temperature scaling into a nonlinear scaler is to use as temperature the output of a neural network. This idea has been explored in several works \citep{tian2020,tomani2022parameterized,joy2023sample}. Another nonlinear extension of temperature scaling is based on using a mixture of temperature-scaled models \citep{zhang2020mix}.

\subsection{The history of temperature scaling}

It seems that the first use of temperature scaling was by \citet{cox1958two}, who used it as a simple parametric test of calibration for binary classifiers. His null hypothesis was $\beta = 1$, which indeed implies perfect calibration. \citet{cox1958two} also mentioned that, when $\hat{\beta} > 1$, ``\textit{probabilities show the right general pattern of variation, but do not vary enough}". On the other hand, when $\hat{\beta} \in (0,1)$, ``\textit{probabilities vary too much}". Cox's \citeyearpar{cox1958two} test and its variants are still widely used in medical research \citep{stevens2020validation}.

The use of a temperature parameter together with a softmax layer is probably an old and common practice. It was particularly popularised by \citet{hinton2015distilling}. For neural language models, one of the early uses of a temperature was in the influential blog post of \citet{karpathyblog} on the successes of recurrent neural nets.
\section{Does higher temperature mean higher uncertainty?}
\label{sec:entropy}

The general motivation for tempering models is that it helps control their level of uncertainty. In this section, we investigate formalisations of this claim. We choose and fix some input $\bx \in \mathcal{X}$. A key motivating property of temperature scaling, already mentioned by \citet{guo2017calibration}, is its limiting behaviour. When $\beta \rightarrow 0$,  $p_\beta(y | \bx)$ converges to the uniform distributions over its support. On the other hand, when $ \beta \rightarrow \infty$, $p_\beta(y | \bx)$ converges to a point mass at the maximum probability class when $p(y | \bx)$ is unimodal (for a proof, see e.g. \citealp{girardin2016escort}, Proposition 1). Thus, by scaling the temperature, we can interpolate between the most uncertain model (the uniform distribution), and the most certain one (the point mass). We will see that this interpolation happens monotonically in a fairly general sense.

\subsection{Entropy is an increasing function of the temperature}

The fact that $\beta \mapsto H(p_\beta)$ decreases was recently proven by \citet[Proposition A.5]{dabah2025temperature} who showed that its derivative was nonpositive. We revisit this result here and provide an alternative and shorter proof, relying on the properties of the log-partition function (Lemma \ref{prop:logpartition}). This proof, inspired by standard statistical physics computations, highlights that the entropy is directly related to the variance of the logits.
\begin{proposition} 
\label{prop:ent_derivative}
For all $\beta >0$,
\begin{equation}
        \frac{dH(p_\beta(y | \bx))}{d \beta}  = - \beta \textup{Var}_{y \sim p_\beta(y | \bx)} ( z_y ) \leq  0.
        \end{equation}
\end{proposition}
As we saw in Proposition \ref{prop:likelihood}, the variance of the logits is also related to the second derivative of the likelihood, and hence, to the Fisher information. Therefore, an interpretation of this result is that \emph{the entropy will vary faster when there is a lot of information in the calibration set}.

While the (Shannon-Gibbs) entropy is the most common uncertainty measure in machine learning, many others exist. We will see that most reasonable measures of uncertainty, temperature scaling behaves monotonically.

\subsection{Beyond the Shannon-Gibbs entropy: majorisation}

Majorisation is a strong measure of dispersion for discrete distributions, orginally introduced in the influential book of \citet{hardy1952inequalities}. The notion has been applied (in particular to derive new inequalities) in many subfields of applied mathematics \citep{marshall2011inequalities} or statistical physics \citep{sagawa2022entropy}.

 There are many different equivalent definitions of majorisation (see e.g. \citealp{marshall2011inequalities}, p.14), but a compelling one is the following.

\begin{definition}
	Let $\bpi = (\pi_1,...,\pi_K)$ and $\boldsymbol{\rho} = (\rho_1,...,\rho_K)$ be vectors in $\mathbb{R}^K$. We say that $\bpi$ majorises $\boldsymbol{\rho}$, and denote it $\bpi \prec \boldsymbol{\rho}$ when, for all  convex functions $\phi: \mathbb{R}  \rightarrow \mathbb{R}$,
	\begin{equation}
		\sum_{k=1}^K \phi(\pi_k ) \leq \sum_{k=1}^K \phi(\rho_k ).
	\end{equation}
\end{definition}

\textbf{Why is majorisation a good measure of dispersion?} It is instructive to look at a few examples of convex functions. Taking $\phi(\pi) = \pi \log \pi$ means that  $\bpi \prec \boldsymbol{\rho}$ implies $H(\bpi) \leq H(\boldsymbol{\rho})$. This means that majorisation is a stronger measure of dispersion than entropy. Actually, $\bpi \prec \boldsymbol{\rho}$  also implies that $\bpi $ has a smaller Renyi entropy than $\boldsymbol{\rho}$, and the same result is true for many generalisations and variants of entropy (see e.g. \citealp{kvaalseth2022entropies}).

It turns out that the effect of temperature on majorisation has been known for a while. The following result was originally derived by \citet{marshall1965norms} to bound the condition numbers of some matrices, and is mentioned and proven by \citet[p. 189]{marshall2011inequalities}. \citet[Theorem 4.1]{dabah2025temperature} proved the exact same result, but did not highlight its connection to majorisation theory.

\begin{theorem} \label{th:majorisation} Let $\bpi \in \Delta_K$. If $0<\beta_1<\beta_2$. Then,
	\begin{equation}
		\bpi_{\beta_1}\prec \bpi_{\beta_2}.
	\end{equation}
\end{theorem}
This result directly implies that $\beta \mapsto H(p_\beta (y | \bx))$ is a decreasing function, not only for the Shannon entropy but for all other entropies that are weaker than majorisation. For instance, the function $\beta \mapsto \max_y \log p_\beta (y |\bx)$ is increasing.

\textbf{The impact of temperature on class proportions.} The previous results formalise that, for a given $\bx$, the conditional distribution $p_\beta (y | \bx)$ gets more and more spread out when $\beta$ shrinks. Surprisingly, this does not imply that the class proportions $p_\beta (y) = \mathbb{E}_{p(\bx)} [p_\beta(y | \bx)]$ will also get more and more spread out. Indeed, averaging is known to break majorisation \citep[Chapter 6]{marshall2011inequalities}, except under very strong conditions (that do not hold in practical classification settings). Therefore, one must remain aware that changing the temperature may have an unexpected impact on class proportions. We will see next another quite unexpected behaviour of temperature scaling, when applied to the conditionals of an autoregressive model.

\subsection{For language models, increasing the temperature can decrease the entropy}
\label{sec:llms_ent}

The previous results are fairly intuitive and grant that temperature scaling mostly behaves ``as it should'' for classification. However, we will see here that it can have an unexpected on language models.
\begin{figure}[t]
    \centering
    \includegraphics[width=0.5\linewidth]{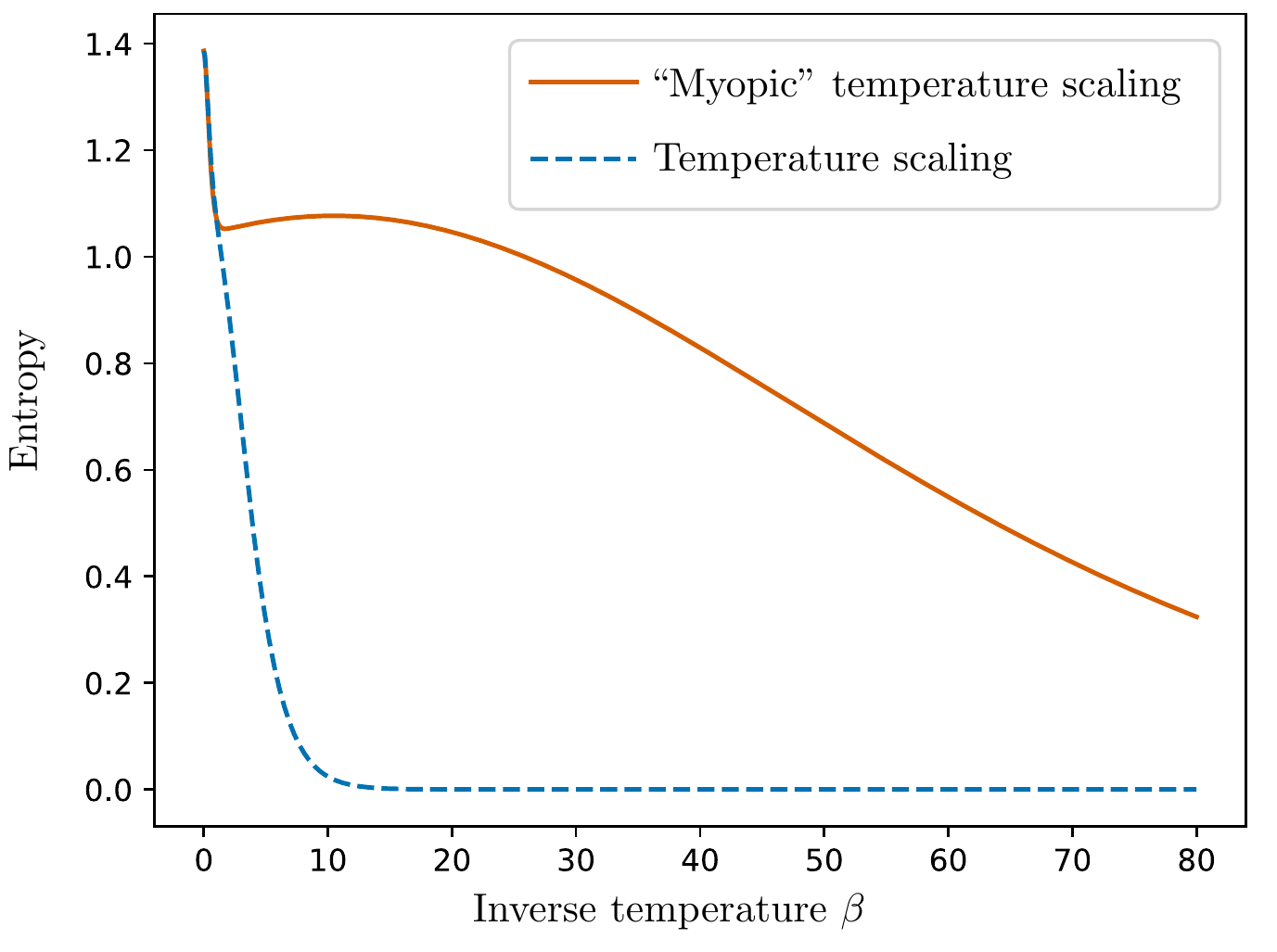}
    \caption{Entropies of two tempered versions of a toy language model. Standard temperature scaling has the expected monotonic effect on entropy, but is impractical for LLMs. The ‘‘myopic" version, widely used in practice, has a much less intuitive behaviour.}
    \label{fig:llm}
\end{figure}
Most state-of-the-art language models are autoregressive models of the form
\begin{equation}
    p(x_1, \ldots,x_T) = p(x_1) \prod_{t=1}^{T-1} p(x_{t+1}|x_t),
\end{equation}
where $x_1, \ldots, x_T$ are tokens living in a finite vocabulary $\mathcal{X}$.
Standard temperature scaling of the joint model is intractable in this case, since computing 
\begin{equation}
    p_\beta (x_1, \ldots,x_T) \propto p(x_1)^\beta \prod_{t=1}^{T-1} p(x_{t+1}|x_t)^\beta,
\end{equation}
would involve dealing with a normalising constant that is a sum over $|\mathcal{X}|^ T$ terms. A much simpler alternative, widely used in practice, is to apply temperature scaling independently to all decoding steps, leading to the distribution
\begin{equation}
    p_\beta'(x_1, \ldots,x_T) = p_\beta(x_1) \prod_{t=1}^{T-1} p_\beta(x_{t+1}|x_t),
\end{equation}
called \emph{myopic temperature scaling} in \citet{shih2023long}.

Of course, in general, there is no reason to expect that $p_\beta = p_\beta'$. Should this be a source of concern? \citet{shih2023long} provide a tractable approximation of $p_\beta$ called \emph{long horizon temperature scaling} and compare this approximation to $p_\beta'$. Based on these experiments and on a toy example, they convincingly argue that tuning the temperature of the approximation of $p_\beta$ leads to a much better quality/diversity trade-off than tuning the temperature of $p_\beta'$.

While standard temperature scaling $p_\beta$ has the predictable behaviour described previously (most reasonable measures of diversity will increase as the temperature increase), we provide here a counter-example that highlights that the entropy  of $p_\beta'$ can be nonmonotonic!

\begin{figure*}[t]
    \centering
    \includegraphics[width=0.6\linewidth]{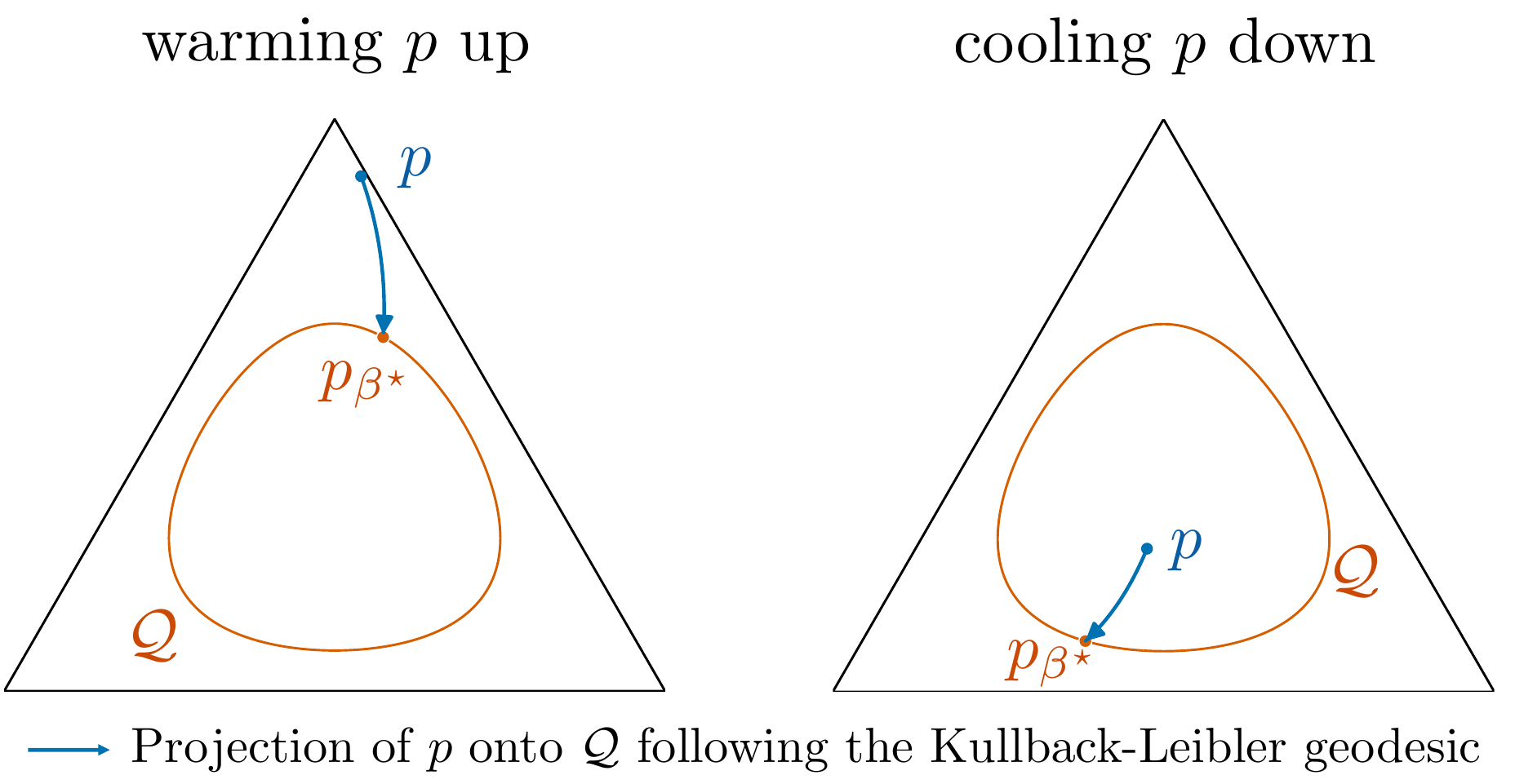}
    \caption{\label{fig:simplex}Temperature scaling as an information projection. In both cases, the initial model $p$ is projected onto the set $\mathcal{Q}$ of distributions of constant entropy $h^\star=0.9$. The blue arrow is the path of all tempered models between $p$ and its projection $p_{\beta^\star}$. This path can be interpreted as the geodesic between $p$ and $p_{\beta^\star}$. \emph{(Left)} The initial distribution $p=(0.01, 0.09, 0.9)$ has a lower entropy than $h^\star$, and is warmed up towards $p_{\beta^\star}$ with $\beta^\star \approx 0.37$. \emph{(Right)} The initial distribution $p=(0.4, 0.35, 0.25)$ has a higher entropy than $h^\star$, and is cooled down towards $p_{\beta^\star}$ with $\beta^\star \approx 3.98$.}
\end{figure*}

Consider the simple toy autoregressive model on a binary vocabulary $\mathcal{X}=\{0,1\}$
\begin{equation}
    p(x_1,x_2) = p(x_1)p(x_2|x_1),
\end{equation}
with $p(x_1)= \mathcal{B}(x_1 | \pi)$ and $p(x_2 | x_1)= \mathcal{B}(x_2 | \rho_{x_1})$ for some $\pi, \rho_0, \rho_1 \in (0,1)$. Surprisingly, for some values of the parameters, $\beta \mapsto H(p'_\beta)$ does not always decrease. Indeed, Figure \ref{fig:llm} shows that when $\pi = 0.51$, $\rho_0 = 0.01$, and $\rho_1 = 0.51$, the entropy of myopic temperature scaling decreases when $\beta$ small, but then starts to increase before finally decreasing again. Why is this the case? It is instructive to decompose the joint entropy using the chain rule:
\begin{equation}
     H(p'_{\beta}) = H(p_{\beta}(x_1)) +  \mathbb{E}_{p_\beta(x_1)} [H(p_{\beta}(x_2| x_1))].
\end{equation}
The entropy of $p_{\beta}(x_1)$ is decreasing so the monotonicity of the joint entropy will mostly depend on the conditional entropy of $x_2$ given $x_1$. When $\beta$ becomes smaller it becomes more and more certain that $x_1 = 0$. However, the conditional entropy of $x_2$ is much larger when $x_1 = 0$ than when $x_1 = 1$, so there will be a regime for which increasing $\beta$ will \emph{increase} the entropy.

This mechanism suggests that similar non-monotonic effects may arise whenever early-token uncertainty strongly conditions later-token entropy. More details on this experiment can be found in Appendix \ref{app:llm}.
\section{A geometric interpretation}
\label{sec:geometry}
We provide here geometric insights about how changing the temperature allows to travel along the simplex. None of these results are new, but they were, to the best of our knowledge, not discussed in this context.

\subsection{Temperature scaling as an information projection}

We saw in the previous section that temperature scaling allows to move from the initial model towards models with higher or lower entropy. There is a precise geometric interpretation of this fact, that we describe here.

Again, we fix some input $\bx \in \mathcal{X}$. The purpose of temperature scaling was to alter the initial model $p(y | x)$ to make it more or less uncertain, i.e. to increase or decrease its entropy. 
A geometric approach to this problem would be to select an entropy level $h^\star >0$ that we would like to reach, and find the distribution closest to $p(y | x)$ that has this given entropy. For instance, if $h^\star \geq H(p)$ then we are looking for a model less confident than $p$. Considering the Kullback-Leibler divergence as a measure of closeness, this becomes the following \emph{information projection} problem:
\begin{equation}
\label{eq:info_proj}
    q^\star(y |\bx) \in \textup{argmin}_{q \in \mathcal{Q}} \textup{KL} (q(y |\bx) || p(y |\bx)),
\end{equation}
where $\mathcal{Q} = \{q  \text{ s.t. } H(q) =  h^\star\}$. It turns out that the solution to this problem is exactly temperature scaling.

\begin{theorem}
\label{thm:info_proj:unique}
    If $p(y |\bx)$ is unimodal and has full support, then the information projection problem \eqref{eq:info_proj} has a unique solution $q^\star$, and there exists $\beta^\star\geq 0$ such that $q^\star(y |\bx) = p_{\beta^\star}(y|\bx).$
\end{theorem}
This result, illustrated on Figure \ref{fig:simplex}, means that \emph{we can interpret a tempered model as the closest model with the required amount of uncertainty}.
A version of this theorem dates back, at least, to \citet{sgarro1978informational}. A recent proof using Lagrange multipliers can be found in \citet[Theorem 1]{girardin2016escort}, who also discuss what happens when there are several modes or when the distribution does not have full support.

\Cref{thm:info_proj:unique} is also a well-known result in the statistical physics literature, where it is typically known in terms of its dual formulation (in the Lagrange sense). Indeed, recalling that in the statistical physics perspective of Section \ref{sec:statphys} $p(y|\bx)$ is the Boltzmann-Gibbs distribution with energies $E_y$ and at temperature $\beta=1$, opening up the expression of the KL divergence we can write:
\begin{align}
    \textup{KL} (q(y |\bx) || p(y |\bx))  = -H( q(y |\bx) )-\mathbb{E}_{y \sim q(y|\bx)}[\log p(y|\bx)] = -H(q(y |\bx))+\mathbb{E}_{y \sim q(y|\bx)}[ E_y ]+\log{\mathcal{Z}}(1,\bz).
\end{align}
Hence, minimising the $\textup{KL} (q || p)$ at fixed entropy $H(q)=h^{\star}$ is equivalent to maximising the entropy $H(q)$ at fixed ``internal energy'' $\mathbb{E}_{q}[E_{y}]=u^{\star}$, which is precisely the original construction of the canonical ensemble in \cite{boltzmann1885moglichkeit}, later formalised by \cite{gibbs1902elementary}. Indeed, entropy and energy are conjugate variables in thermodynamics, with the linear relationship between them typically used as a definition for the inverse-temperature in equilibrium statistical physics. From this perspective, the assumption of unimodality of $p(y|\bx)$ in Theorem \ref{thm:info_proj:unique} is equivalent to the existence of a unique ground state (i.e. unique minimal energy $E_{y}=-z_{y}$). This connection between information geometry and statistical physics was pioneered by \cite{jaynes1957information}.  

Note this result does not hold under other metrics than the Kullback-Leibler divergence, and temperature scaling will generally not be optimal. \citet[Proposition 4]{girardin2016escort} provide a similar result for the reverse Kullback-Leibler divergence.

\subsection{Tempering along geodesics}

What does it mean, geometrically, to move from one temperature to another? A starting point is to notice that, for any $0<\beta_1 < \beta < \beta_2$, 
\begin{equation}
    p_{\beta} \propto p^{\beta}= p^{\alpha \beta_1 + (1-\alpha) \beta_2}=p_{\beta_1} ^\alpha p_{\beta_2}  ^{1-\alpha},
\end{equation}
with $ \alpha = (\beta_2-\beta)/(\beta_2 - \beta_1) \in (0,1)$. This means that $p_{\beta}$ is the normalised geometric mean of  $p_{\beta_1}$ and $p_{\beta_2}$. Such normalised geometric means, sometimes called ‘‘products of expert'', ‘‘geometric mixtures'', or ‘‘logarithmic pool'', have many interesting properties (see e.g. \citealp{amari2007integration,carvalho2023,razafindralambo2026beyond}). 

 In particular, in information geometry, if we see the family of discrete distributions as an exponential family, the curve $\beta \in (\beta_1, \beta_2) \mapsto p_{\beta}$ is the geodesic from $p_{\beta_1}$ to $p_{\beta_2}$ (see, e.g. \citealp{amari2016information}, Section 2.4). This means that moving from one temperature to another will be moving along a geodesic, as seen in Figure \ref{fig:simplex}.

\section{Temperature scaling is the only accuracy-preserving linear scaler}
\label{sec:consistency}
We prove here another simple characterisation of temperature scaling, as the only linear scaler that does not alter ``hard" model predictions. We say that a scaler is \emph{accuracy-preserving} when it does not change the model predictions, i.e. when the $\argmax$ of the scaled predictions is the same as the one of the original predictions. Note that the  $\argmax$ of a vector is defined as the set of indices that attain the max of the vector (in particular it is not necessary a single label).

\citet{guo2017calibration} attributed the superiority of temperature scaling over matrix scaling to its accuracy-preserving nature. This property was then further studied by \citet{zhang2020mix}, who coined the term ``accuracy-preserving", and \citet{rahimi2020intra}, who both tried to design accuracy-preserving scalers that are more flexible than temperature scaling. These works, as well as other accuracy-preserving scalers proposed thereafter \citep{esaki2024accuracy,zhang2025instance}, gain this additional flexibility by making scalers nonlinear. Here, we prove that this nonlinearity is necessary:  \emph{it is impossible for a linear scaler to be both order-preserving and more flexible than temperature scaling}.

Our main tool will be the following result, that characterises argmax-invariant linear functions. To the best of our knowledge, this result has not appeared in the literature, and may be of independent interest.

\begin{theorem}
\label{th:accuracy_preserving}
    Let $\bW \in \mathbb{R}^{K \times K}$ and  $\bb \in \mathbb{R}^K$ such that, for all $\bz \in \mathbb{R}^K$ we have
    \begin{equation}
    \label{eq:invariance}
        \textup{argmax}(\bW \bz+\bb) = \textup{argmax}(\bz).
    \end{equation}
Then, there exist $\ba \in \mathbb{R}^K$, $\beta> 0$, and $\gamma \in \mathbb{R}$ such that
\begin{equation}
    \bW = \beta \bI_K + \bun_K \ba ^T, \text{ and } \bb = \gamma \bun_K.
\end{equation}
\end{theorem}
The proof is available in Appendix \ref{app:accuracy_preserving}. There are two different kinds of linear scalers:
\begin{itemize}[noitemsep,leftmargin=*]
    \item the matrix scaling model $\tilde{p}_{\bW,\bb}$, defined in Equation \eqref{eq:matscaling}, which is linear in the logits,
    \item the Dirichlet calibration model $\check{p}_{\bW,\bb}$, defined in Equation \eqref{eq:dirichlet}, which is linear in the log-probabilities.
\end{itemize}
While the two scalers are not formally equivalent, they both contain temperature scaling as a submodel. The following result shows that this is the only submodel that is accuracy-preserving in both cases.

\begin{corollary}
\label{cor:accuracy_preserving} 
   Let  $\bW \in \mathbb{R}^{K \times K}$ and  $\bb \in \mathbb{R}^K$ such that $\tilde{p}_{\bW,\bb}$ (respectively $\check{p}_{\bW,\bb}$) is accuracy-preserving.
   
   Then, there exists $\beta \geq 0$ such that $\tilde{p}_{\bW,\bb} = p_{\beta}$ (respectively $\check{p}_{\bW,\bb} = p_{\beta}$) .
\end{corollary}

The proof is a direct application of Theorem \ref{th:accuracy_preserving}, and is available in Appendix \ref{app:accuracy_preserving2}.

This result clarifies why attempts to design more expressive accuracy-preserving scalers must inevitably depart from linearity (at least when the initial model has a good accuracy). Conversely, it explains why temperature scaling often performs competitively despite its extreme simplicity: among all linear recalibrations, it already exhausts the available degrees of freedom once accuracy is fixed. This result is therefore both an argument in favour of temperature scaling, and also one in favour of nonlinear scalers. 
\section{Discussion}
\label{sec:discussion}
We revisited the elementary properties of temperature scaling, and hope that this work can serve as an invitation to study this simple and popular model.

For probabilistic classifiers, temperature scaling behaves exactly as intended. Increasing the temperature monotonically increases uncertainty in the strong sense of majorisation. This justify its interpretation as a principled uncertainty control tuner. In contrast, we showed that these guarantees do not extend automatically to autoregressive language models as they are used in practice. This observation challenges the common assumption that temperature acts as a reliable diversity knob at the sequence level.

Beyond uncertainty, our results also clarify the structural role of temperature scaling among calibration methods: it is the only linear calibration method that preserves hard predictions. Consequently, any attempt to design more expressive accuracy-preserving scalers must necessarily rely on nonlinear transformations.

One question that remains wide open is \emph{why does temperature scaling performs so well?} \citet{berta2025structured} showed that even in very simple cases, the optimal scaler is generally nonlinear, thus it is surprising to see a model that simple work that well.
Another avenue of future research would be to investigate the practical relevance of our counter-example for real-life LLMs.

Despite its theoretical nature, our work leads to two practical take-home messages:
\begin{itemize}
    \item \emph{For LLMs, temperature can have counterintuive effects on the diversity,} which may explains why other methods to play with text stochasticity are sometimes preferred \citep{holtzman2020curious}, and is in line with the empirical findings of \citet{peeperkorn2024temperature}.
    \item \emph{To calibrate very accurate classifiers, temperature scaling should be preferred to matrix scaling and Dirichlet calibration,} since temperature scaling is the only way to implement accuracy-preservation into these models.
\end{itemize}

\section*{Acknowledgements}
Both authors thank Chengwei Liang for spotting an error in an earlier version of this preprint.

BL was supported by the French government, managed by the National Research Agency (ANR), under the France 2030 programme with the reference ``ANR-23-IACL-0008'' and the Choose France - CNRS AI Rising Talents program. 

PAM was also supported by the French government, through the France 2030 programme managed by the ANR, with the reference number ``ANR-23-IACL-0001''.

\bibliographystyle{abbrvnat}
\bibliography{bibliography}

\newpage
\appendix
\section*{\LARGE Appendix}
\section{Proof that using the logits and using the log-probabilities are equivalent}
\label{app:logitsvsloprobs}
With the notations of Section \ref{sec:whatisTS}, since  $\boldsymbol{\pi} = \text{Softmax} (\bz )$, we find
\begin{align}
    \text{Softmax} (\beta \log \bpi ) &= \text{Softmax} \left(\beta \left(\bz - \log\left(\sum_{y=1}^Ke^{z_y}\right)  \bun_K \right)  \right) \\
    &=  \text{Softmax} (\beta \bz ),
\end{align}
using the fact that the Softmax is invariant to adding the same constant (here $-\beta \log (e^{z_1}+ \ldots + e^{z_K})$) to all coordinates.

\section{Proof of Lemma \ref{prop:logpartition} (derivatives of the log-partition function)}
\label{app:logpartition}
The log-partition function is defined as
\begin{equation}
    \log \mathcal{Z}(\beta, \bz) = \log \left( \sum_{y=1}^K e^{\beta z_y}\right).
\end{equation}
Thus, its first derivative is
\begin{equation}
\label{eq:firstderofZ}
    \frac{d \log \mathcal{Z}(\beta, \bz)}{d\beta } = \frac{ \sum_{y=1}^K z_y e^{\beta z_y}}{ \mathcal{Z}(\beta, \bz)} = \sum_{y=1}^K z_y {\color{blue}\underbrace{\frac{e^{\beta z_y}}{\mathcal{Z}(\beta, \bz)}}_{p_\beta(y| \bx)}}=  \mathbb{E}_{y \sim p_\beta(y | \bx)} [ z_y ].
\end{equation}
The second derivative can also be seen as expectation, using the fact that $\frac{d}{d\beta} p_\beta(y | \bx) = \frac{d \log p_\beta(y | \bx)}{d \beta }  p_\beta(y | \bx)$,
\begin{equation}
\label{eq:secondderofZ}
    \frac{d^2 \log \mathcal{Z}(\beta, \bz)}{d\beta^2 } = \frac{d}{d\beta}\sum_{y=1}^K z_k p_\beta(y | \bx) = \sum_{y=1}^K z_k \frac{d \log p_\beta(y | \bx)}{d \beta }  p_\beta(y | \bx) = \mathbb{E} \left[ z_y \frac{d \log p_\beta(y | \bx)}{d \beta } \right ].
\end{equation}
Now, using the definition of $p_\beta(y | \bx)$ and Equation \eqref{eq:firstderofZ} yields
\begin{equation}
    \frac{d \log p_\beta(y | \bx)}{d \beta } = z_y - \frac{d \log \mathcal{Z}(\beta, \bz)}{d\beta } =  z_y - \mathbb{E}_{y \sim p_\beta(y | \bx)} [ z_y ].
\end{equation}
Plugging this back in Equation \eqref{eq:secondderofZ} finally leads to
\begin{equation}
    \frac{d^2 \log \mathcal{Z}(\beta, \bz)}{d\beta^2 } =  \mathbb{E} \left[ z_y (z_y - \mathbb{E}_{y \sim p_\beta(y | \bx)} [ z_y ] ) \right ] = \mathbb{E}_{y \sim p_\beta(y | \bx)}[z_y^2] -  \mathbb{E}_{y \sim p_\beta(y | \bx)}[z_y] ^2 = \textup{Var}_{y \sim p_\beta(y | \bx)} ( z_y ).
\end{equation}
\section{Proof of Proposition \ref{prop:likelihood} (properties of the likelihood)}
\label{app:likelihood}
The negative log-likelihood is equal to
\begin{equation}
    \mathcal{L}(\beta) =  -\frac{1}{n} \sum_{i=1}^n \log  p_{\beta} ( y_i | \bx_i) =  -\frac{1}{n} \sum_{i=1}^n \left( \beta z_{y_i} - \log \mathcal{Z}(\beta, \bz_i) \right).
\end{equation}
Thus, using Lemma \ref{prop:logpartition}, we find
\begin{equation}
\label{eq:likelihoodz}
    \frac{d\mathcal{L} (\beta) }{d \beta } =    -\frac{1}{n} \sum_{i=1}^n \left( z_y - \mathbb{E}_{y \sim p_\beta(y | \bx)} [ z_y ] \right).
\end{equation}
Now, since $\log p_\beta(y | \bx_i) = \beta z_{y_i} - \log \mathcal{Z}(\beta, \bz_i)$, we can write $ z_{y_i} = (1/\beta) (\log p_\beta(y | \bx_i) + \log \mathcal{Z}(\beta, \bz_i))$, and plugging that in Equation \eqref{eq:likelihoodz}, we find
\begin{equation}
 \frac{d \mathcal{L}(\beta)}{d\beta} =\frac{1}{\beta} \left( \underbrace{-\frac{1}{n} \sum_{i=1}^n  \log p_\beta (y_i | \bx_i)}_\textup{cross-entropy}    -  \underbrace{\left(-\frac{1}{n} \sum_{i=1}^n\mathbb{E}_{y \sim p_\beta(y | \bx_i)} [ \log p_\beta(y |\bx_i)]\right)} _\textup{expected cross-entropy}\right).
    \end{equation}    
The second derivative is a direct consequence of the second derivative of the log-partition (Lemma \ref{prop:logpartition}):
        \begin{equation}
        \frac{d^2 \mathcal{L} (\beta)}{d\beta^2}  = \frac{1}{n} \sum_{i=1}^n \textup{Var}_{y \sim p_\beta(y | \bx_i)} ( z_y ).
    \end{equation}   
\section{Proof of Proposition \ref{prop:ent_derivative} (monotonicity of the entropy)}
\label{app:monotonicity}
The entropy is equal to 
\begin{align}
    H(p_\beta(y | \bx)) &= -\sum_{y=1}^K p_\beta(y | \bx) \log p_\beta(y | \bx) =  -\sum_{y=1}^K p_\beta(y | \bx) \left( \beta z_y - \log \mathcal{Z}(\beta, \bz) \right) \\ &= - \beta {\color{blue}\underbrace{\mathbb{E}_{y \sim p_\beta(y | \bx)}[z_y]}_{\frac{d \log \mathcal{Z}(\beta, \bz)}{d\beta } \mbox{ \small (Lemma \ref{prop:logpartition})}}} +  \log \mathcal{Z}(\beta, \bz).
\end{align}
Thus, using Lemma \ref{prop:logpartition} again, we find
\begin{equation}
   \frac{d H(p_\beta(y | \bx))}{d\beta} = - \beta \frac{d^2 \log \mathcal{Z}(\beta, \bz)}{d\beta^2 } - \frac{d \log \mathcal{Z}(\beta, \bz)}{d\beta } + \frac{d \log \mathcal{Z}(\beta, \bz)}{d\beta } = - \beta \frac{d^2 \log \mathcal{Z}(\beta, \bz)}{d\beta^2 } = - \beta  \textup{Var}_{y \sim p_\beta(y | \bx)} ( z_y ).
\end{equation}

\section{More details on the toy language model}
\label{app:llm}

\begin{figure}[t]
    \centering
    \includegraphics[width=0.48\linewidth]{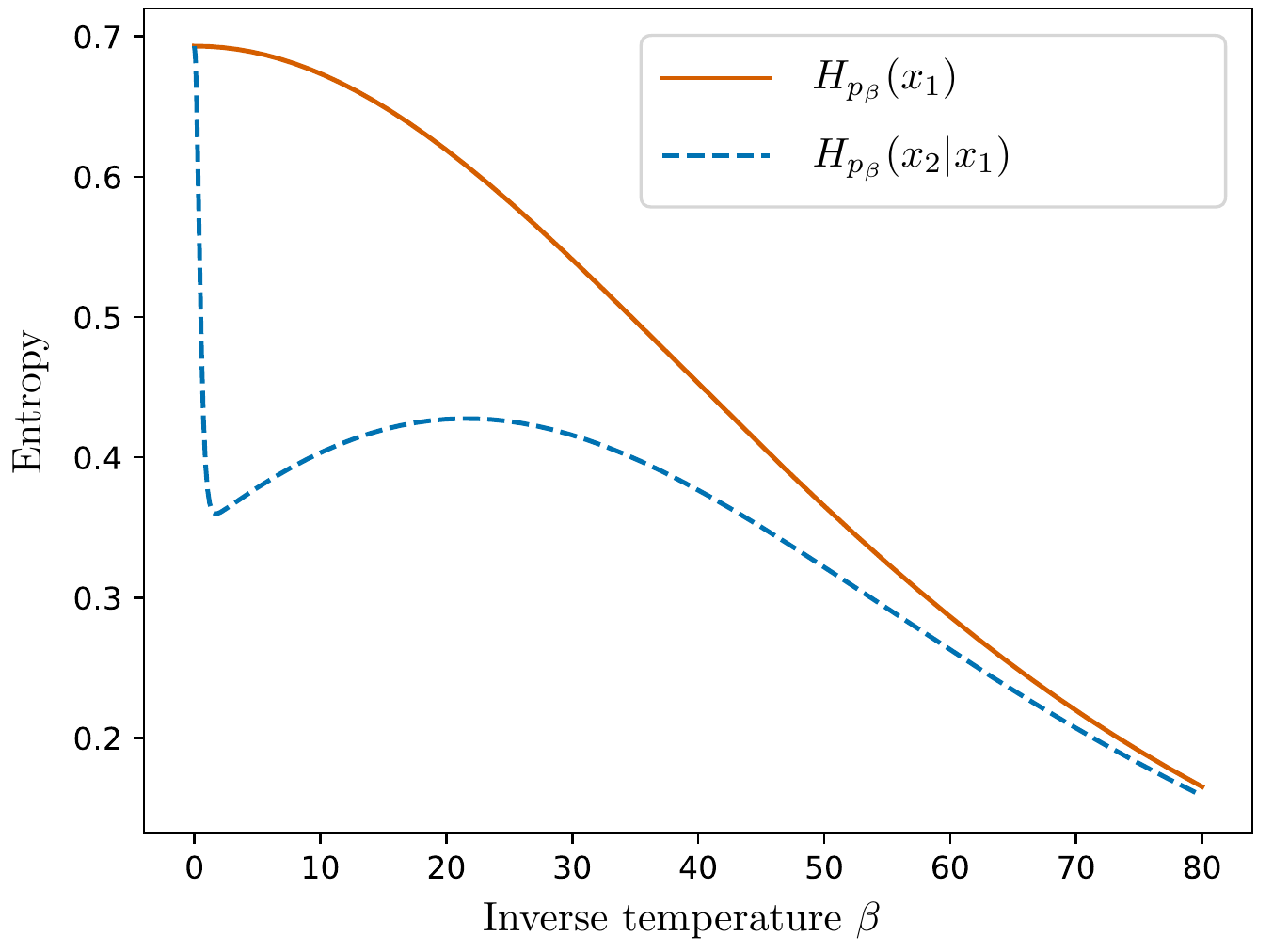}
    \includegraphics[width=0.48\linewidth]{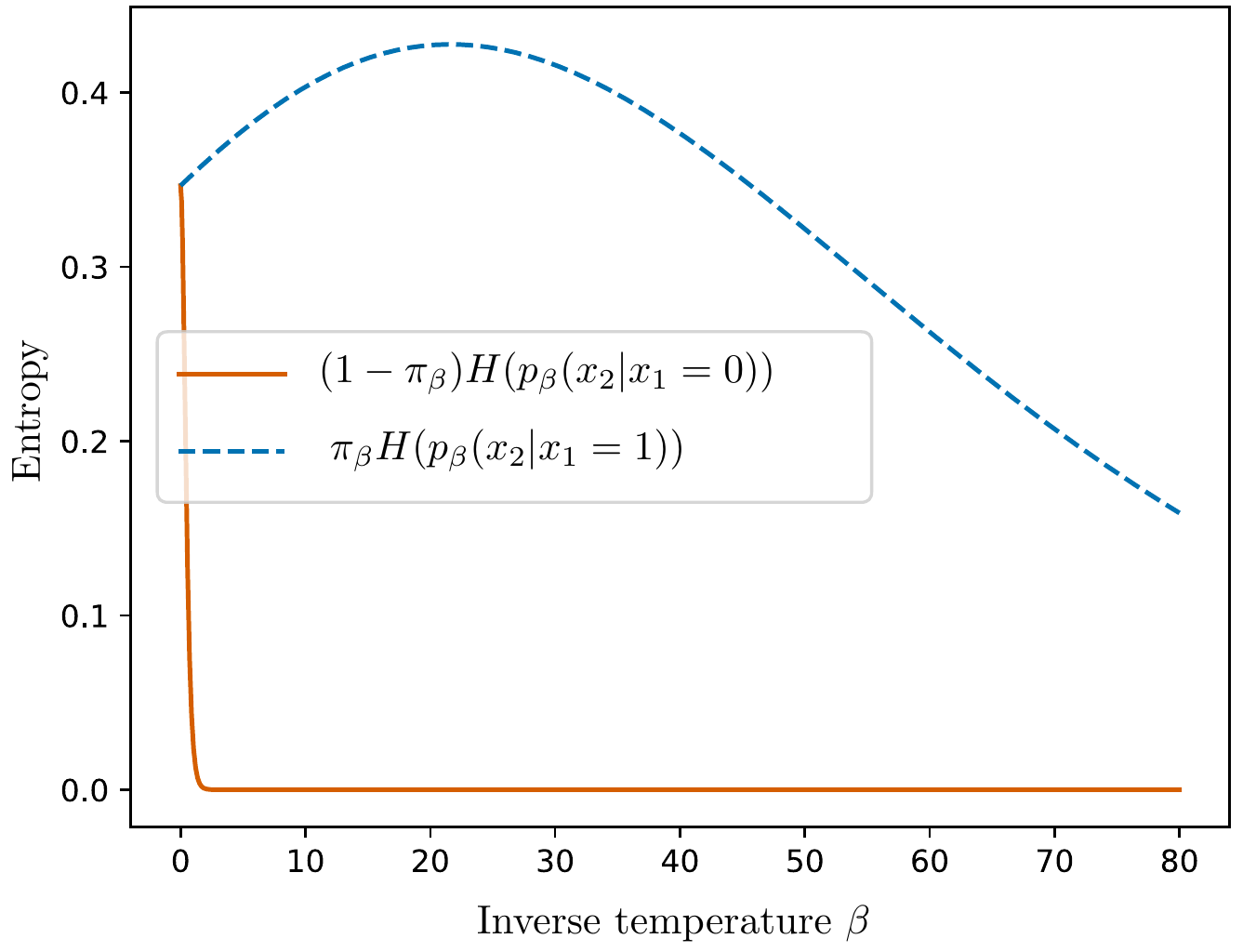}
    \caption{Terms involved in the computation of the entropy of the myopically scaled model. \emph{(Left)} The marginal and conditional entropies of Equation \ref{eq:app_chainrule1}. \emph{(Right)} The two terms of the decomposition of Equation \eqref{eq:app_chainrule2}. The source of nonmotonicity is the term $\pi_\beta H(p_{\beta}(x_2| x_1 = 1)$, that is the product of an increasing and a decreasing functions.}
    \label{fig:llm}
\end{figure}

In the context of Section \ref{sec:llms_ent}, the entropy of the myopic model can be decomposed using the chain rule as
\begin{equation}
\label{eq:app_chainrule1}
     H(p'_{\beta}) = H(p_{\beta}(x_1)) +  \mathbb{E}_{p_\beta(x_1)} [H(p_{\beta}(x_2| x_1))],
\end{equation}
where $H(p_{\beta}(x_1))$ is the marginal entropy of the first token, that we will denote $H_{p_{\beta}}(x_1)$, and  $\mathbb{E}_{p_\beta(x_1)} [H(p_{\beta}(x_2| x_1))]$ is the conditional entropy, that we will denote $H_{p_{\beta}}(x_2 | x_1)$. According to Proposition \ref{prop:ent_derivative}, $H_{p_{\beta}}(x_1)$ is decreasing. Since $p(x_1)= \mathcal{B}(x_1 | \pi)$ and $p(x_2 | x_1)= \mathcal{B}(x_2 | \rho_{x_1})$, the conditional entropy can be decomposed as
\begin{equation}
\label{eq:app_chainrule2}
    \mathbb{E}_{p_\beta(x_1)} [H(p_{\beta}(x_2| x_1))] =   {\color{blue}\underset{\searrow }{(1- \pi_\beta)}} {\color{green}\underset{\searrow }{H(p_{\beta}(x_2| x_1 = 0))}} +  {\color{yellow}\underset{\nearrow }{\pi_\beta}}{\color{red} \underset{\searrow }{H(p_{\beta}(x_2| x_1 = 1))}}.
\end{equation}
The two entropies of the conditional distributions are decreasing (again, because of Proposition \ref{prop:ent_derivative}). However, since $\pi = 0.55$, $\pi_\beta$ is an increasing function of $\beta$, which explains that the conditional entropy can actually be locally increasing.

\section{Proof of Theorem \ref{th:accuracy_preserving} (characterisation of linear accuracy-preserving functions)}
\label{app:accuracy_preserving}
\begin{theorem}
    Let $\bW \in \mathbb{R}^{K \times K}$ and  $\bb \in \mathbb{R}^K$ such that, for all $\bz \in \mathbb{R}^K$ we have
    \begin{equation}
    \label{eq:invarianceapp}
        \textup{argmax}(\bW \bz+\bb) = \textup{argmax}(\bz).
    \end{equation}
Then, there exist $\ba \in \mathbb{R}^K$, $\beta> 0$, and $\gamma \in \mathbb{R}$ such that
\begin{equation}
    \bW = \beta \bI_K + \bun_K \ba ^T, \text{ and } \bb = \gamma \bun_K.
\end{equation}
\end{theorem}

\begin{proof}
The roadmap of the proof goes as follows: first, we will prove that $\bb$ has the desired form, then we will prove that $\bW$ has the desired form (by studying separately the cases $K=2$ and $K \geq 3$). Finally, we will show that $\beta>0$.

Applying Equation \eqref{eq:invarianceapp} to $\bz=0$ leads to $\text{argmax}(\bb)= \text{argmax}(\boldsymbol{0}_K)= \{1, \ldots, K \} $, therefore all coefficients of $\bb$ need to be equal, and thus there exists $\gamma \in \mathbb{R}$ such that $\bb = \gamma \bun_K$.

We now turn to $\bW$. The fact that $\bb = \gamma \bun_K$ implies that, for all $\bz \in \mathbb{R}^K$, 
\begin{equation}
\label{eq:invariance_linear}
    \text{argmax}(\bW \bz+\bb)= \text{argmax}(\bW \bz) = \text{argmax}(\bz),
\end{equation}
which will be our main tool to study the structure of $\bW$.

We begin with the case $K=2$. Let $\bW \in \mathbb{R}^{2 \times 2}$ that verifies \eqref{eq:invariance_linear}, and let $\alpha_1 = w_{21}$, $\alpha_2 = w_{12}$, and $\beta = w_{11}- \alpha_1$. The fact that $\text{argmax}(\bW \bun_2) = \text{argmax}(\bun_2)= \{1,2\}$ implies that
\begin{equation}
    w_{11} + w_{12} = w_{21} +w_{22} \implies w_{22} = w_{11} + w_{12} - w_{21} = \beta + \alpha_1 + \alpha_2 - \alpha_1 = \beta + \alpha_2.
\end{equation}
Thus, $\bW = \beta \bI_K + \bun_K \alpha ^T$.

Let us now treat the case $K\geq 3$. For any $j \in \{1,\ldots,K\}$, let us denote $\be_j\in \mathbb{R}^K$ the vector whose only nonzero coefficient is the $j$th one, which is equal to $1$. Let $\bW$ be a matrix that satisfies \eqref{eq:invariance_linear}. Let $i,j,k$ be three distinct indices in $\{1, \ldots, K\}$. Applying the condition to the vector $\be_i + \be_j + \be_k$ leads to
\begin{equation}
    w_{ii} + w_{ij} + w_{ik} = w_{ji} + w_{jj} + w_{jk} = w_{ki} + w_{kj} + w_{kk},
\end{equation}
and applying it to  $\be_i + \be_k$ leads to
\begin{equation}
    w_{ii} + w_{ik} = w_{ki} + w_{kk}.
\end{equation}
Subtracting these equations implies that $w_{ij} = w_{kj}$. This means that, for each column $j$, all elements but the diagonal one will have the same value, that we will denote $\alpha_j$. 

Let us now use condition \eqref{eq:invariance_linear} with the vector $\bun_K$. This implies that 
\begin{equation}
    w_{11} + \sum_{j \neq 1} \alpha_j =  w_{22} + \sum_{j \neq 2} \alpha_j = \ldots = w_{KK} + \sum_{j \neq K} \alpha_j.
\end{equation}
Subtracting $\alpha_1 + \ldots + \alpha_K$ from all terms leads then yields
\begin{equation}
    w_{11} - \alpha_1 = \ldots =  w_{KK} - \alpha_K.
\end{equation}
Let us denote this quantity $\beta$. Putting the pieces together, we have proven that $\bW = \beta \bI_K + \bun_K \ba ^T$.

The only thing left to prove is the positivity of $\beta$. This follows from the fact that $\text{argmax}(\bW \be_1) = \text{argmax}(\be_1) = \{1\}$. Indeed, this implies that $w_{11}>w_{21}$ i.e. that $\beta + \alpha_1 > \alpha_1$, leading to $\beta>0$.
\end{proof}

\section{Proof of Corollary \ref{cor:accuracy_preserving} (temperature scaling is the only accuracy-preserving linear scaler)}
\label{app:accuracy_preserving2}
\begin{corollary}
   Let  $\bW \in \mathbb{R}^{K \times K}$ and  $\bb \in \mathbb{R}^K$ such that $\tilde{p}_{\bW,\bb}$ (respectively $\check{p}_{\bW,\bb}$) is accuracy-preserving.
   
   Then, there exists $\beta \geq 0$ such that $\tilde{p}_{\bW,\bb} = p_{\beta}$ (respectively $\check{p}_{\bW,\bb} = p_{\beta}$) .

\end{corollary}
\begin{proof}
    Let us focus on matrix scaling (the proof for Dirichlet calibration is essentially the same). Let  $\bW \in \mathbb{R}^{K \times K}$ and  $\bb \in \mathbb{R}^K$ such that $\tilde{p}_{\bW,\bb}$ is accuracy-preserving. Because of Theorem \ref{th:accuracy_preserving}, there exist $\ba \in \mathbb{R}^K$, $\beta> 0$, and $\gamma \in \mathbb{R}$ such that
$
    \bW = \beta \bI_K + \bun_K \ba ^T, \text{ and } \bb = \gamma \bun_K
$. Thus, for all $\bz \in \mathbb{R}^K$
\begin{equation}
    \textup{Softmax} (\bW \bz+\bb) = \textup{Softmax} ( \beta \bz + (\ba^T \bz + \gamma)\bun_K) = \textup{Softmax} ( \beta \bz),
\end{equation}
using the fact that the softmax is invariant to adding a vector with identical coordinates, therefore $\tilde{p}_{\bW,\bb} = p_{\beta}$.
\end{proof}

\end{document}